\documentclass[a4paper,fleqn]{cas-sc}

\usepackage[authoryear]{natbib}
\usepackage{amsthm}
\usepackage{subfigure}
\newtheorem{theorem}{Theorem}
 
\newtheorem{example}{Example}
\newtheorem{definition}{Definition}
\newtheorem{lemma}{Lemma} 
\newtheorem{remark}{Remark} 
\usepackage[normalem]{ulem}
\def\tsc#1{\csdef{#1}{\textsc{\lowercase{#1}}\xspace}}
\tsc{WGM}
\tsc{QE}
\tsc{EP}
\tsc{PMS}
\tsc{BEC}
\tsc{DE}

\usepackage{hyperref}
\hypersetup{
    colorlinks=true,
    linkcolor=blue,
    filecolor=magenta,      
    urlcolor=blue,
    citecolor=blue
}

\begin{document}
\let\WriteBookmarks\relax
\def\floatpagepagefraction{1}
\def\textpagefraction{.001}
\shorttitle{Reduced-Order Dynamics with Singularity via Latent-Augmented Neural ODEs}


\title [mode = title]{Learning Reduced-Order Dynamics with Singularity via Latent-Augmented Neural Ordinary Differential Equations}

\author[1]{Xiaorui Wang}
\fnmark[1]
\ead{242081101003@lut.edu.cn}
\credit{Methodology, Writing – original draft, Software, Validation, Visualization}
\affiliation[1]{organization={Lanzhou University of Technology},
            addressline={School of Automation and Electrical Engineering}, 
            city={Lanzhou},
            postcode={730050}, 
            state={Gansu},
            country={China}}

\author[2]{Yu Zhou}
\fnmark[1]
\ead{yuzhou@hiroshima-u.ac.jp}
\credit{Conceptualization, Methodology, Writing – original draft, Investigation, Supervision}
\affiliation[2]{organization={Hiroshima University},
            addressline={Graduate School of Advanced Science and Engineering}, 
            city={Higashihiroshima},
            postcode={739-0046}, 
            state={Hiroshima},
            country={Japan}}

\author[3]{Wenjie Mei}
\cormark[1]
\ead{mei.wenjie@nju.edu.cn}
\credit{Conceptualization, Methodology, Funding acquisition, Writing – review and editing, Supervision, Project administration}
\affiliation[3]{organization={Nanjing University - Suzhou Campus},
            addressline={School of Robotics and Automation}, 
            city={Suzhou},
            postcode={215163}, 
            state={Jiangsu},
            country={China}}

\author[4]{Dongzhe Zheng}
\ead{dz5992@princeton.edu}
\credit{Resources, Writing – review and editing}
\affiliation[4]{organization={Princeton University},
            addressline={Department of Mechanical and Aerospace Engineering}, 
            city={Princeton},
            postcode={08544}, 
            state={New Jersey},
            country={United States}}

\author[2]{Yang Bai}
\ead{yangbai@hiroshima-u.ac.jp}
\credit{Resources, Writing – review and editing}

\author[2]{Masaaki Nagahara}
\ead{nagam@hiroshima-u.ac.jp}
\credit{Resources, Writing – review and editing}

\cortext[1]{Corresponding author}

\fntext[1]{These authors contributed equally to this work.}


\begin{abstract}
This paper addresses the issue of self-intersecting trajectories (in phase space) in industrial reduced-order modeling and proposes the Latent-Augmented Neural Ordinary Differential Equations (LA-NODEs) framework. From the perspective of artificial intelligence, the proposed method augments conventional neural ordinary differential equations to enhance model expressiveness, enabling the representation of conflicting vector fields that may arise in reduced-order systems, thereby improving learning accuracy. Through theoretical analysis, the underlying mechanism of the framework is established, and a condition for determining the minimum required augmentation dimension is derived. From the perspective of engineering applications, the effectiveness of the proposed method is validated on the reduced-order system of two representative industrial models, namely an interior permanent magnet synchronous motor (IPMSM) drive and a distributed energy system (DES). Experimental results demonstrate that the proposed method can recover system features that are difficult to capture using conventional approaches and achieve superior performance in terms of prediction accuracy and modeling fidelity, thereby providing an effective approach for high-precision data-driven modeling of complex industrial systems.

\end{abstract}

\begin{keywords}
Reduced-order system\sep Neural ordinary differential equations\sep Singularity\sep  Industrial system\sep Augmentation
\end{keywords}

\maketitle

\section{Introduction}
In practical industrial processes, high-dimensional, strongly nonlinear, and computationally expensive models are often intractable for real-time optimization and model predictive control tasks \cite{industry1,industry2}. Therefore, model order reduction is commonly employed to derive low-dimensional and computationally efficient surrogate models \cite{MOR1, MOR2}. This type of reduced-order system modeling problem is typically addressed using data-driven approaches to construct surrogate dynamical models.

Data-driven system modeling is becoming significant in modern engineering, with critical applications spanning autonomous robotics to complex energy networks, for example \cite{DATA1, DATA2,DATA3}. As universal function approximators, neural networks (NNs) offer the capacity to learn intricate nonlinear mappings directly from raw data, leading to their widespread application as powerful tools for system identification \cite{pillonetto2025deep, TII_NN, NN4,shen2026improved}. In contrast to traditional modeling based on the discrete approximation of physical laws, continuous-time modeling possibly provides a more realistic description of a system's evolution, offering superior fidelity and temporal consistency \cite{continuous_system1, continuous_system2}. An important work in this domain was the introduction of Neural Ordinary Differential Equations (NODEs) in \cite{NODEs}, which interprets Residual Networks (ResNets) \cite{RESNET} as the discretization of continuous-time dynamical systems. This paradigm shift has enabled the smooth integration of deep learning with the rich theoretical framework of differential equations. The inherent extensibility of NODEs has promoted extensive research, resulting in diverse architectures \cite{ConvNODE, GNODE,odetrans}. Apart from these foundations, recent work in \cite{CSODE} introduced ControlSynth NODEs (CSODEs), which enhance the representational capacity of standard NODEs and demonstrate improved performance in modeling complex physical systems. 

Existing studies have shown that data-driven approaches can construct equivalent dynamical models using only limited observable data, without requiring full access to the system states \cite{REMARK,RE2}. However, such reduction procedures, while improving computational efficiency, often come at the cost of sacrificing the geometric structure, dynamical properties, and numerical robustness of the original model. Specifically, when the high-dimensional system dynamics are projected onto a low-dimensional observation space, non-injective mappings may arise, where multiple latent states correspond to the same observation. At the trajectory level, this manifests as self-intersections. Such phenomena can undermine the physical consistency and interpretability of the resulting model.

Although many NODE variants improve performance, they mainly enhance input dimensionality, with little attention to output dimensionality. This introduces a potential limitation: Almost all NODE variants rely on numerical ODE solvers to integrate dynamics, and they inherently assume that the vector field is a single-valued function of the state. Consequently, neither NODEs nor CSODEs can accurately capture system trajectories exhibiting self-intersection: scenarios where a single state point admits multiple distinct directions of motion, which is defined as "singularity" in this work. While introducing explicit time dependence can theoretically disambiguate these directions, it requires the model to learn a mapping with extreme temporal precision, which is often computationally intractable. Furthermore, in many industrial applications, the latent variables responsible for singularity are not solely time-dependent; thus, treating time as the only additional input is insufficient to resolve such topological ambiguities.

To address these challenges, the Augmented NODEs (ANODEs) framework was proposed in \cite{ANODE, ANODE2}. Nevertheless, ANODEs were primarily designed for classification tasks, and the existing literature lacks a comprehensive theoretical analysis regarding the necessity and selection of the augmented dimension for dynamical system modeling. 
Motivated by these limitations, this paper proposes the Latent-Augmented Neural Ordinary Differential Equations (LA-NODEs) framework. By augmenting dimensions, LA-NODEs resolve vector field ambiguities at singularities, thereby enabling the framework to recover the true system dynamics. We establish a formal connection between the required augmented dimension and the geometric properties of trajectories with at least one singularity. 

\textbf{The main contributions of this paper} are summarized as follows:
\begin{itemize}
\item 
We demonstrate that, when learning reduced-order dynamical systems using the classical NODE framework, the approximation error cannot be made arbitrarily small in the presence of singularities in the observed data. In particular, we show that autonomous NODEs inevitably incur a nonzero training error due to the multi-valued nature of the underlying dynamics at singular points, leading to a fundamental learning limitation.

\item 
We propose a Latent-Augmented Neural ODE (LA-NODE) framework that introduces learnable latent dynamics to lift the system into a higher-dimensional space. This augmentation resolves the multi-valuedness induced by singularities and, in theory, enables the model to achieve zero training error, thereby overcoming the fundamental limitations of autonomous NODEs.

\item 
For practical training, we investigate how to select the minimal latent dimension required to resolve singularities. We show that the effectiveness of LA-NODE depends on whether the augmented space has sufficient degrees of freedom to represent the local velocity structure at singular points. Specifically, we derive a lower bound on the required augmentation dimension based on the rank of the velocity set, which characterizes the intrinsic complexity of the singularity. This result provides a principled guideline for augmented dimension selection, revealing a trade-off between model expressivity and computational efficiency.

\end{itemize}

The remainder of this paper is organized as follows. Section \ref{sec2} presents the main problems to be addressed in this paper. Section \ref{sec3} formalizes the challenge of singularities and identifies the theoretical constraints that preclude these architectures from modeling trajectories exhibiting such phenomena. It also clarifies the fundamental principles of the dimension augmentation method. Section \ref{sec4} details the proposed LA-NODEs framework and presents the related theoretical derivations. Sections \ref{sec5} and \ref{sec6} evaluate the performance of the framework and the validity of the theoretical findings through industrial case studies and sensitivity analyses of the augmented dimension, respectively. Finally, Section \ref{sec7} concludes the study and gives potential directions for future research.

\section{Problem Formulation}\label{sec2}
Neural Ordinary Differential Equations (NODEs) provide a  framework for approximating continuous-time dynamical systems from observed data. Consider a true system governed by
\begin{equation}\label{autody}
\dot{x}(t) = F(x(t)),
\end{equation}
where $x(t) \in \mathbb{R}^n$ is the state vector and $F\colon \mathbb{R}^n  \to \mathbb{R}^n$ represents the unknown vector field. NODEs parameterize the vector field using a neural network $f_{\theta}$ with parameters $\theta$:
\begin{equation} \label{eq:neural_dynamic}\dot{\hat{x}}(t) = f_{\theta}(\hat{x}(t), t).
\end{equation}

In many practical scenarios, the objective is to learn autonomous dynamics, such as~\eqref{autody}. In such cases, the time-dependent driving terms in~\eqref{eq:neural_dynamic} are absent, and the model reduces to the autonomous form:
\begin{equation}\label{EQAUTONODE}
\dot{\hat{x}} = f_\theta(\hat{x}).
\end{equation}

Moreover, in practice, the full system state is often unobservable, and only partial or low-dimensional measurements can be obtained. Consequently, the high-dimensional state $x$ in~\eqref{autody} is mapped into a lower-dimensional observation space through
$
y=h(x),
$
where $h \colon \mathbb R^n\rightarrow\mathbb R^p$ with $p\le n$. The corresponding observation dynamics satisfy
\begin{equation}\label{EQAUTOANODE}
\dot y
=
\frac{\partial h(x)}{\partial x}F(x) := J(x) F(x),
\end{equation} 
where $y\in\mathbb R^p$ denotes the observed state.

In this case, the dynamics of the reduced-order system are simplified, possibly losing some of the information from the original system. Specifically, multiple internal states may correspond to the same observed data, but they may exhibit different dynamics or velocities. In the case of reduced-order systems, their dynamics need to be learned from the observed data. However, NODE fails to fully capture the intrinsic dynamics of the reduced-order system, making it difficult to recover the complete dynamic behavior of the system \cite{iccar}.

The issue of learning the dynamics of reduced-order systems highlights the limitations of, for example, the NODE framework in capturing the system's underlying dynamics, which has motivated the detailed investigation presented in this paper. 

\section{Learning Limitations of Autonomous NODEs}
\label{sec3}
This section formally defines the problem above and presents a lemma, establishing that autonomous NODEs~\eqref{EQAUTONODE} are fundamentally unable to represent such data consistently, and must meet a nonzero training error. These constitute part of the contributions of this work.

\subsection{Formal Description of Problem}
Consider the true system~\eqref{autody}. In reduced-order modeling, the system is often described using a lower-dimensional representation via a mapping $h \colon \mathbb{R}^n \to \mathbb{R}^p$ (where $p \leq n$), yielding the observed state $y(t) \in \mathbb{R}^p$:
$$y = h(x).$$
Although the trajectory in the phase space $\mathbb{R}^n$ is unique and non-intersecting, its projection on $\mathbb{R}^p$ may exhibit trajectory intersections. 


This phenomenon is determined jointly by the observation mapping \(h\) and the underlying dynamics \(F\). The observation map may exhibit local rank degeneracy, characterized, for \(p\le n\), by
\[
\operatorname{rank}(J(x))<p,
\]
or may be non-injective, so that distinct internal states can share the same observation: 
\[
h(x_a)=h(x_b),\qquad x_a\neq x_b.
\]
More generally, even when \(J(x)\) has full row rank, such ambiguity may arise whenever the observation dimension is lower than the state dimension. A closed reduced-order dynamics in the observation space requires the projected vector field \(J(x)F(x)\) to be constant along each fiber of \(h\). If, for some \(x_a\neq x_b\),
\[
h(x_a)=h(x_b),
\qquad
J(x_a)F(x_a)\neq J(x_b)F(x_b),
\]
then the same observation is associated with different instantaneous velocities, and hence no single-valued autonomous dynamics \(\dot y=f(y)\) can represent the observation-space evolution.

\begin{definition}\label{d1}
A reduced-order system trajectory $y(t)=h(x(t))$ possesses a \textbf{singularity} at $y^\star\in\mathbb{R}^p$ if  there exist distinct internal states $x_a, x_b \in \mathbb{R}^n$ ($x_a \neq x_b$) such that $h(x_a) = h(x_b) = y^\star$, while the induced velocities in the observation space are distinct:
$$
J(x_a)F(x_a) \neq J(x_b)F(x_b).
$$
 \end{definition}
At a singularity $y^\star$, there exist distinct times $t_1 \neq t_2$ such that $y(t_1) = y(t_2)$, yet the observed velocities differ: $\dot{y}(t_1) \neq \dot{y}(t_2)$. Note that, under the same initial condition $x_0$, if $t_1 = t_2$, one must have $\dot{y}(t_1) =\dot{y}(t_2)$ due to the inherent property of true system.
Therefore, theoretically, such reduced-order dynamics cannot be represented by a well-defined autonomous neural network since no unique velocity can be assigned to the observation point $y^\star$.

To verify this, we consider directly using a NODE with parameters $\omega$ to approximate the trajectory of $y$ without access to the internal dynamics. Its formulation is given by
\[
\dot{\hat{y}} = \hat{f}_{\omega}(\hat{y}).
\]

When applying such a NODE-based framework, which focuses on learning bijective behaviors of the mapping from state to velocity, to learn reduced-order systems with singularities, the non-injectivity property hinders the correct learning of the system's dynamics. This issue is illustrated through the following example:

\begin{example}\label{ex1}
Consider a three-dimensional dynamical system
\begin{equation}
\begin{cases}
\dot{x}_1 = x_2, \\
\dot{x}_2 = -x_1 + 0.3 x_3, \\
\dot{x}_3 = 1,
\end{cases}
\end{equation}
where a reduced-order representation is given by the observation $y = \begin{bmatrix} x_1, x_2 \end{bmatrix}^\top$. Discarding $x_3$ makes $h$ non-injective. 

As illustrated in Fig.~\ref{fig.ex}, the NODE variant: CSODE \cite{CSODE} learns a trajectory that deviates from the true trajectory. The reduced-order system's trajectory forms a spiral in three-dimensional space, while after reduction, it is represented by the black-colored trajectory in the figure, which introduces a singularity, hindering CSODE's ability to capture the system's intrinsic dynamics accurately.

\end{example}

\begin{figure}
    \centering    
\includegraphics[width=0.5\textwidth]{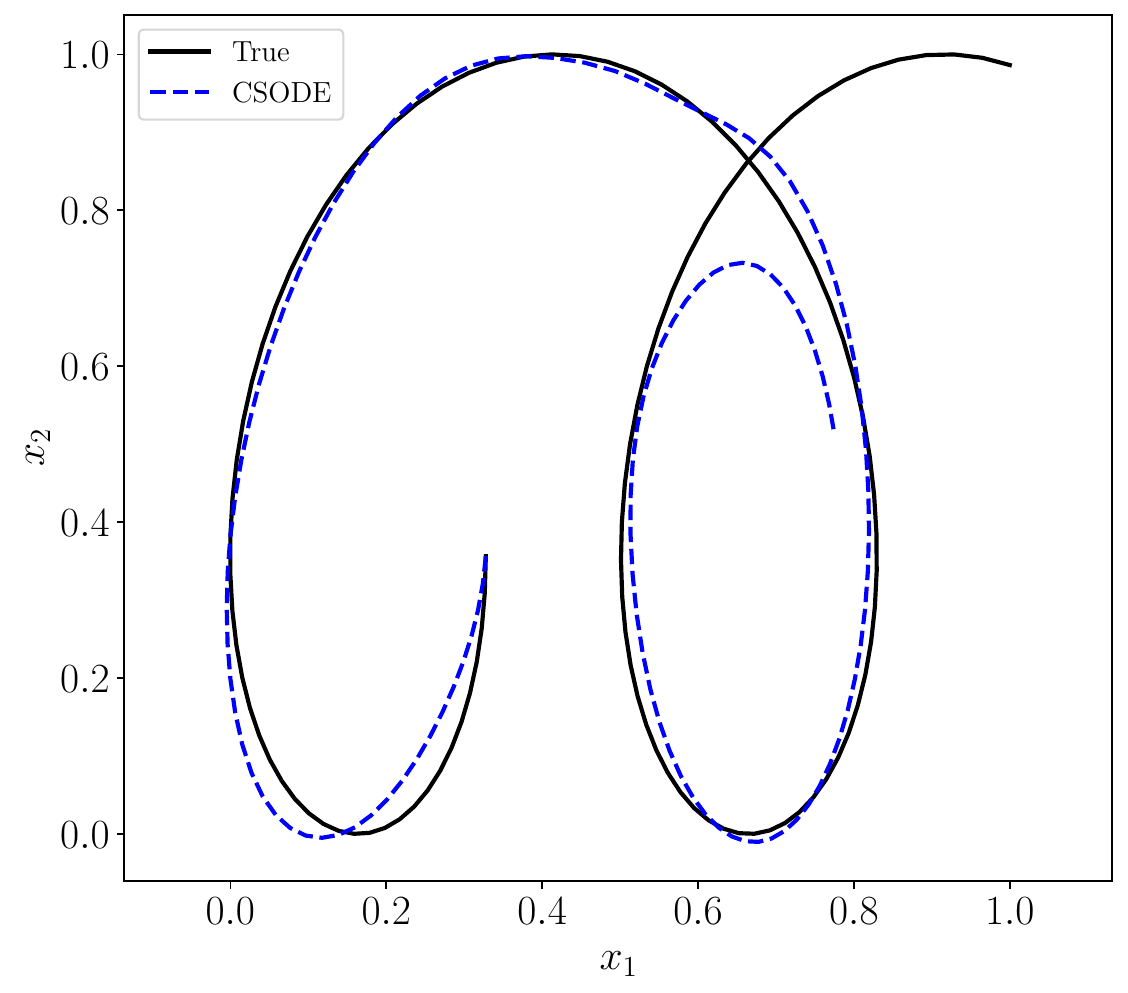}
    \caption{Comparison of learned and true trajectories for example system.}\label{fig.ex}
\end{figure}

It can be observed that when model order reduction induces singularities, NODE variants that do not account for the internal dynamics cannot accurately capture the underlying system dynamics. The negative influence of singularity on the learning performance of NODEs is presented below.

\subsection{Training Limitation Resulting from Singularity}

For a straightforward autonomous NODE that does not consider internal dynamics, we first present a lemma that characterizes the impact of singularities on its training and variants. Note that, for brevity, in the following, we use $y$ rather than $\hat{y}$ to represent the state of neural networks. 

\begin{lemma} \label{lemma_training_nozero}
    Let $\mathcal{D} = \{(y_i, v_i)\}_{i=1}^N$ be a dataset of $N$ state-velocity observations, where $v_i$ denotes the approximation of $\dot{y}$ at the instant $i$. Let $\mathcal{X} = \{u_1, u_2, \dots, u_M\}$ be the set of unique states in $\mathcal{D}$, and $\mathcal{I}_j = \{i \mid y_i = u_j \}$ (here $j \in \{1, \dots, M\}$) be the nonempty set of indices where the states are identical. For any autonomous NODE $\dot{y}= f_\omega(y)$ with $f_\omega \colon \mathbb{R}^p \to \mathbb{R}^p$, the mean squared error (MSE) for training is lower bounded by
    $$
    \mathcal{L}(\theta) := \frac{1}{N} \sum_{i=1}^N \| f_\omega(y_i) - v_i \|^2 \geq \frac{1}{N} \sum_{j=1}^M \sum_{i \in \mathcal{I}_j} \| \bar{v}_j - v_i \|^2, 
    $$
    where $\bar{v}_j = \frac{1}{|\mathcal{I}_j|} \sum_{i \in \mathcal{I}_j} v_i$ is the average velocity for state $u_j$.
\end{lemma}

\begin{proof}
    The total loss $\mathcal{L}(\theta)$ can be divided by gathering all observations corresponding to the same unique state $u_j$:
    $$
    \mathcal{L}(\theta) = \frac{1}{N} \sum_{j=1}^M \sum_{i \in \mathcal{I}_j} \| f_\omega(u_j) - v_i \|^2.
    $$
    For each $j \in \{1, \dots, M\}$, let $z_j = f_\omega(u_j)$ denote the single velocity vector assigned by the autonomous vector field to the state $u_j$. We consider the local error sum $S_j(z_j) = \sum_{i \in \mathcal{I}_j} \| z_j - v_i \|^2$. Expanding this term around the mean $\bar{v}_j$:
    $$
    \| z_j - v_i \|^2 = \| (z_j - \bar{v}_j) + (\bar{v}_j - v_i) \|^2 
    $$
    $$
    \implies \; \| z_j - v_i \|^2 = \| z_j - \bar{v}_j \|^2 + \| \bar{v}_j - v_i \|^2 + 2(z_j - \bar{v}_j)^\top(\bar{v}_j - v_i).
    $$
    Summing over $i \in \mathcal{I}_j$, the cross-product term vanishes due to $\sum_{i \in \mathcal{I}_j} (\bar{v}_j - v_i) = 0$. Thus,
    $$
    S_j(z_j) = |\mathcal{I}_j| \| z_j - \bar{v}_j \|^2 + \sum_{i \in \mathcal{I}_j} \| \bar{v}_j - v_i \|^2.
    $$
    Since $\| z_j - \bar{v}_j \|^2 \geq 0$, the minimum of $S_j(z_j)$ is strictly lower bounded when $z_j = \bar{v}_j$. 
    Substituting this back into the total loss:
    \begin{align}
        \mathcal{L}(\theta) & = \frac{1}{N} \sum_{j=1}^M \left( |\mathcal{I}_j| \| f_\omega(u_j) - \bar{v}_j \|^2 + \sum_{i \in \mathcal{I}_j} \| \bar{v}_j - v_i \|^2 \right) \\
       & \geq \frac{1}{N} \sum_{j=1}^M \sum_{i \in \mathcal{I}_j} \| \bar{v}_j - v_i \|^2 , 
    \end{align}
    This lower bound is determined solely by the dataset $\mathcal{D}$ and is independent of the parameters $\omega$ or the specific architecture of $f_\omega$. 
\end{proof}

Lemma~\ref{lemma_training_nozero} reveals a fundamental limitation in the expressive capacity of the NODE architecture when dealing with singularities. In the case of reduced-order systems with singularities, the error between the velocity learned by the architecture and the velocity at the singular state is unavoidable. The magnitude of this error depends on the differences in velocities observed at the singular state. As a result, the training error is bounded by these velocity differences and cannot be reduced to zero, thus limiting the framework's ability to accurately capture the internal dynamics of the system.

The issue originates from the inability of the NODE architecture to distinguish between states that should be different but have become indistinguishable due to the reduction process. If these states can be distinguished in an appropriate higher-dimensional augmented space, the NODE architecture may be able to accurately learn the dynamics of the reduced-order system. Therefore, augmenting the observation space to capture latent dynamic features provides an effective solution to this problem. This approach is discussed in detail in the following sections.

\section{Proposed Method: Latent-Augmented NODEs}\label{sec4}
To address the limitations demonstrated in Lemma~\ref{lemma_training_nozero} and inspired by the work of CSODE~\cite{CSODE}, this paper proposes the \emph{Latent-Augmented Neural ODE (LA-NODE)} framework, which is capable of learning the dynamics of systems with at least one singularity with high precision. The primary-latent architecture can be formulated as a coupled system of differential equations:
\begin{subequations}\label{LA-node}
\begin{align}
\dot{y}(t) &= \tilde{a}_0\begin{bmatrix}y(t) \\ z(t) \end{bmatrix} 
+ \sum_{i=1}^{H}(\tilde{a}_i)^\top 
\tilde{f}_i\left(\tilde{s}_i\begin{bmatrix}y(t)\\z(t) \end{bmatrix}\right),  
\label{LA-node:1} \\
\dot{z}(t) &= g_\phi(y(t), z(t)).
\label{LA-node:2}
\end{align}
\end{subequations}
where $y(t) \in \mathbb{R}^p$ is the observed state, $z(t) \in \mathbb{R}^k$ represents the augmented variables, and $\tilde{a}_0 \in \mathbb{R}^{p \times (p+k)}$, $\tilde{s}_i \in \mathbb{R}^{r \times (p+k)}$, and $\tilde{a}_i \in \mathbb{R}^{b \times p}$ are trainable weight matrices. The activation function $\tilde{f}_i \colon \mathbb{R}^{r} \to \mathbb{R}^b$, and the neural network $g_\phi \colon \mathbb{R}^{p+k} \to \mathbb{R}^k$ governs the evolution of the augmented state.

By introducing the augmented state $z(t)$, the proposed framework aims to resolve the issue induced by singularities and reduce the training loss. Compared with standard NODEs, the learned dynamics can achieve lower loss over a finite time horizon.

\begin{figure}
    \centering    
\includegraphics[width= 0.8\textwidth]{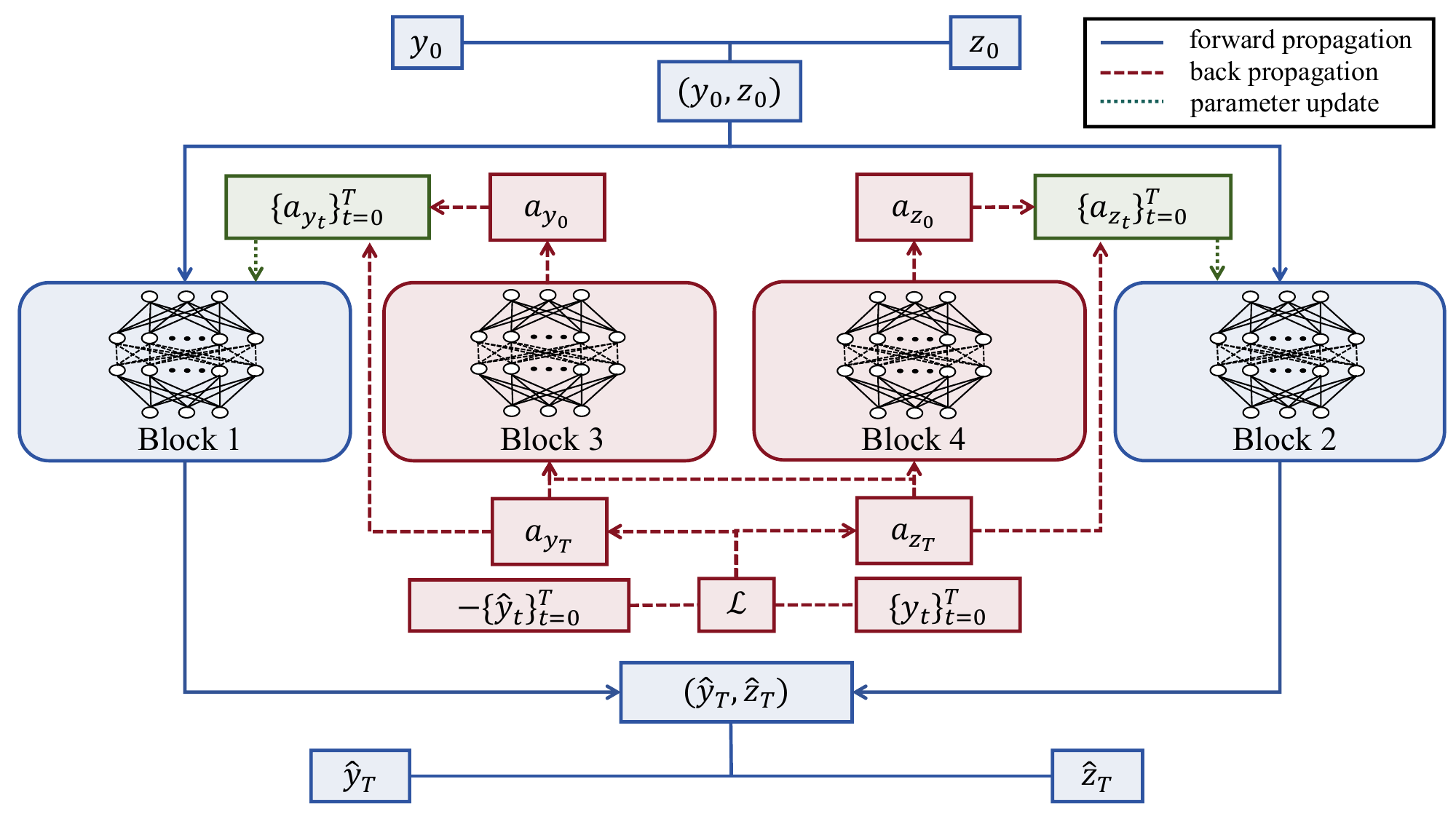}
  \caption{Architecture of the LA-NODE model, including primary neural dynamics (Block 1, Eq.~\eqref{LA-node:1}), latent neural dynamics (Block 2, Eq.~\eqref{LA-node:2}), adjoint primary neural dynamics (Block 3, Eq.~\eqref{adjoint1}), and adjoint latent neural dynamics (Block 4, Eq.~\eqref{adjoint2}). $a_{y_t}$ and $a_{z_t}$ represent the adjoint state defined as Eq.~\eqref{adjointstate}. 
  Blue lines denote forward propagation, red lines denote adjoint (loss-driven) backward propagation, and green lines denote parameter updates as shown in Eq.~\eqref{update}.}
  \label{fig.3}
\end{figure}
More generally, the LA-NODE can be represented in the compact form:
\begin{equation}\label{ald}
\begin{cases}
\dot{y}(t) = \tilde{f}_\theta(y(t), z(t)), \\
\dot{z}(t) = g_\phi(y(t), z(t)),
\end{cases}
\quad \begin{bmatrix}
    y \\ z
\end{bmatrix} \in\mathbb{R}^{p+k}.
\end{equation}

It mitigates the singularity problem encountered by NODEs in learning reduced-order systems by augmenting the state dimension to capture latent variables. The augmented variables $z(t)$ act as a latent workspace that facilitates the dynamics during integration and is discarded after inference.

 The architecture of LA-NODE, as illustrated in Fig.~\ref{fig.3}, in forward propagation, the evolution of the state can be expressed as \eqref{ald}.
As for the backward propagation, the gradients are computed by solving the adjoint differential equations backward in time, following the method in \cite{NODEs}. We first give the definition of the adjoint state, which can be expressed as
\begin{equation}\label{adjointstate}
a_y(t) = \frac{\partial \mathcal{L}}{\partial y(t)}, \, 
\quad 
a_z(t) = \frac{\partial \mathcal{L}}{\partial z(t)}.
\end{equation}
where $\mathcal{L} = \frac{1}{2} \int_0^T \left\| y(t) - y_{\text{true}}(t) \right\|^2 dt$. Under \eqref{adjointstate}, we have
\begin{subequations}\label{adjoint}
\begin{align}
\displaystyle \frac{d a_y(t)}{dt} 
&= - a_y(t)\frac{\partial \tilde{f}_\theta}{\partial y} 
  - a_z(t)\frac{\partial g_\phi}{\partial y} - (y(t) - y_{\text{true}}(t))^\top,
  \label{adjoint1} \\
\displaystyle \frac{d a_z(t)}{dt} 
&= - a_y(t)\frac{\partial \tilde{f}_\theta}{\partial z} 
  - a_z(t)\frac{\partial g_\phi}{\partial z},
  \label{adjoint2}
\end{align}
\end{subequations}
where $a_y(T) = 0,a_z(T)=0$, the loss function depends only on the observed state $y(t)$, while the augmented state $z(t)$ influences the gradient through the coupling terms of the system. Parameters $\theta,\phi$ can be updated by
\begin{equation}\label{update}
\begin{cases}
    \frac{d\mathcal{L}}{d\theta} 
=- \int_T^0 a_y(t)\,\frac{\partial \tilde{f}_\theta(y(t),z(t))}{\partial \theta}\, dt,\\
\frac{d\mathcal{L}}{d\phi} 
=- \int_T^0 a_z(t)\,\frac{\partial g_\phi(y(t),z(t))}{\partial \phi}\, dt.
\end{cases}
\end{equation}
This establishes a complete learning procedure for the proposed LA-NODE framework.

Similar to the previous augment-involved NODE variants (e.g., ANODE  \cite{ANODE, ANODE2}),
    despite the introduction of the augmented state $z(t)$, our learning objective
remains focused exclusively on accurately modeling the evolution of the observed
state $y(t)$.
Accordingly, the loss function is defined only on the $p$-dimensional state
$y(t)$.

\subsection{Trainability of Augmented Model}

The method of LA-NODE enhances the framework by introducing $k$-dimensional augmented variables. Without corresponding ground truth values as supervision, the representations of these variables can be learned by the framework, as clarified in the following lemma.

\begin{lemma}\label{Augmented}
    Let $\mathcal{D} = \{(y_i, v_i)\}_{i=1}^N$ be a dataset where there exists at least one singularity. Although the approximation error for an autonomous NODE in $\mathbb{R}^p$ has a strictly positive lower bound, the augmented system in $\mathbb{R}^{p+k}$ (see Eq. \eqref{ald}) can achieve an approximation error of zero, where $k \geq 1$.
\end{lemma}

\begin{proof}
    Consider the augmented state $\xi(t) = [y(t)^\top, z(t)^\top]^\top \in \mathbb{R}^{p+k}$. For any set of singularities in the dataset $\{(y_i, v_i)\}_{i \in \mathcal{I}}$ where $y_i = y^\star$ for all $i \in \mathcal{I}$, the constraint $f_\omega(y_i) = f_\omega(y_j)$ can be omitted if the augmented states are pairwise distinct, i.e., $z_i \neq z_j$ for all $i \neq j$, where $i, j \in \mathcal{I}$.
    
    Since the evolution of $z(t)$ is governed by the learnable dynamics $\dot{z} = g_\phi(y, z)$, the framework can learn an initial condition $z(0)$ or a trajectory that ensures $\xi_i \neq \xi_j$ for all $i, j$. In this augmented space $\mathbb{R}^{p+k}$, the vector field $\tilde{f}_\theta\colon \mathbb{R}^{p+k} \to \mathbb{R}^p$ is a function defined by
    $$
    \dot{y}(t) = \tilde{f}_\theta(\xi(t)).
    $$
    Under $\xi_i \neq \xi_j$ in the domain of $\tilde{f}_\theta$, the model can satisfy:
    $$
    \tilde{f}_\theta(\xi_i) = v_i \quad \text{and} \quad \tilde{f}_\theta(\xi_j) = v_j
    $$
    for $v_i \neq v_j$. Thus, the lower bound of training error from the non-augmented case can be reduced to zero, i.e., the empirical loss $\mathcal{L}(\theta) \to 0$ as the expressive capacity of $\tilde{f}_\theta$ and $g_\phi$ increases.
\end{proof}

Lemma~\ref{Augmented} demonstrates that state augmentation can resolve the singularity issue. Intuitively, in the reduced-order space, multiple trajectories in the phase space may intersect, leading to different velocities at the same state. By introducing augmented dimensions, the proposed framework lifts the system into a higher-dimensional space, where these intersecting trajectories can be separated. As a result, each augmented state corresponds to a unique velocity, allowing the learned vector field to remain bijective while fitting all observations.

\subsection{Resolving Learning Issue via Practical Augmentation}
The trainability of the augmented model in theory allows the LA-NODE to resolve the singularity issue discussed in Section \ref{sec2}. However, the practical effectiveness of the augmentation method depends on whether the augmented dimension $k$ provides sufficient degrees of freedom to span the manifold of observed velocities at each singularity. Specifically, if a state point $y^\star$ meets $p$ (here, $p$ may be greater than two) dependent velocity vectors, the augmented space has to be large enough to distinguish these $p$ states. This leads to a fundamental requirement regarding the dimensionality of the augmentation:

\begin{theorem}\label{co}
Let $\mathcal{V}(y^\star) = \{ v_1, v_2, \dots, v_M \}$ be the set of distinct velocity vectors associated with a singularity $y^\star$ in $\mathbb{R}^p$. Let velocity rank $r(y^\star) = \dim(\text{span}(\mathcal{V}(y^\star)))$. The LA-NODE can represent these dynamics only if the augmented dimension $k$ satisfies
\begin{equation}
k \ge \max_{y^\star} \; \{ r(y^\star) \}.
\end{equation}
\end{theorem}

\begin{proof}
     Let $\mathcal{F}(y^\star) = \{ \tilde{f}_\theta(y^\star, z) \mid z \in \mathbb{R}^k \}$ be the set of velocities representable by the model at a fixed observation $y^\star$. Since $\tilde{f}_\theta$ is a continuous mapping (typically a neural network), the image $\mathcal{F}(y^\star)$ is a subset in $\mathbb{R}^p$ whose dimension is bounded by the dimension of its latent input $z$. Formally, by the properties of smooth mappings, we have $\dim(\mathcal{F}(y^\star)) \le k$. For the model to capture the true underlying dynamics without error, the set of true velocities $\mathcal{V}(y^\star)$ must be contained within the representable set of the model, i.e., $\mathcal{V}(y^\star) \subseteq \mathcal{F}(y^\star)$. This inclusion implies that the subspace spanned by the true velocities must not exceed the dimensions available in the latent workspace:
     \begin{equation}
     \operatorname{span}(\mathcal{V}(y^\star)) \subseteq \operatorname{span}(\mathcal{F}(y^\star)) \quad \implies \quad r(y^\star) \le k.
     \end{equation}
     The map $\tilde{f}_\theta$  acts as a projection from the augmented latent manifold back to the observed tangent space. If $k < r(y^\star)$, the mapping forces at least $r(y^\star) - k$ dimensions of velocity information to be lost, leading to the scenario established in Lemma \ref{lemma_training_nozero}. Extending this requirement to all singularities in the domain yields $k \ge \max_{y^\star} \{ r(y^\star) \}$, which completes the proof.
\end{proof}

\begin{remark}
Different singularities may possess different velocity ranks, and Theorem~\ref{co} provides a global lower bound determined by the most demanding singularity. Since these ranks are generally unknown in practice, the augmentation dimension $k$ is typically treated as a hyperparameter. Choosing $k$ below the maximum singularity rank leads to the nonzero training-error lower bound characterized in Lemma~\ref{lemma_training_nozero}, whereas an excessively large $k$ mainly increases computational cost and may introduce unnecessary model complexity. Practically, a common strategy is to start from a small value (e.g., $k=1$) and increase it incrementally until satisfactory training performance is achieved.
\end{remark}

To further clarify the statement in Theorem~\ref{co}, we present the following illustrative example: Consider a reduced-order state $y(t) \in \mathbb{R}^3$ with a singularity at $y^\star = \mathbf{0}$ at $t \in \{t_1, t_2, t_3\}$ with velocities (here $M=3$):
    \begin{equation*}v_1 = \begin{bmatrix} 1 \\ 0 \\ 0 \end{bmatrix}, \quad v_2 = \begin{bmatrix} 0 \\ 1 \\ 0 \end{bmatrix}, \quad v_3 = \begin{bmatrix} 1 \\ 1 \\ 0 \end{bmatrix}.
    \end{equation*}
    Since $v_3 = v_1 + v_2$, we have $r(\mathbf{0}) = \dim(\text{span}\{v_1, v_2, v_3\}) = 2$. By Theorem~\ref{co}, an augmented dimension $k < 2$ constrains the representable velocities $\mathcal{F}(y^\star)$ to a manifold of dimension less than $r = 2$, making the non-collinear velocities $\{v_i\}$ inseparable. Conversely, with $k \ge 2$, the framework can assign distinct augmented states $z_i \in \mathbb{R}^k$ to $y^\star$, effectively learning the local dynamics by augmenting the trajectory to $\mathbb{R}^{3+k}$ where the velocity uniqueness is satisfied.





In the following, we present the experimental setup and results to further validate the performance of LA-NODE in addressing the singularity issue in reduced-order systems.
\section{Experiments}\label{sec5}
We select two industrial dynamics: an interior permanent magnet synchronous motor (IPMSM) drive \cite{PMSM} and a distributed energy system (DES) \cite{DES}. Under specific operating conditions, their two reduced-order systems exhibit different singularity properties, including two-trajectory and three-trajectory singularities, which correspond to cases where the singularity point admits two or three distinct velocity directions, respectively.

\subsection{Experimental Setup}

\paragraph{Hardware and Software:} All experiments were conducted on a Windows 11 workstation with an Intel Core i9-14900KF CPU, 16 GB RAM, and an NVIDIA GeForce RTX 4090 GPU (24 GB). The models were implemented in PyTorch 1.11.0 and Python 3.9 using CUDA 11.3 for acceleration.

\paragraph{Metrics and Training:}  In the experiments, all data were standardized and split into training and testing sets with a ratio of 9:1. All models were trained using the Adam optimizer with a learning rate of 0.01. The mean squared error (MSE) is employed as the loss function for both the training and testing phases. The number of training epochs was configured differently for each system: 1000 epochs for the IPMSM system and 700 epochs for the DES. Furthermore, to reduce randomness and enhance the reliability of the results, each experiment was independently repeated five times, and the results were comprehensively analyzed.

\paragraph{Baselines:} To benchmark the proposed LA-NODE, three baseline frameworks are selected: Transformer \cite{Tran}, CSODE \cite{CSODE}, and ANODE \cite{ANODE}. 
\begin{itemize}
    \item \textbf{Transformer}: The experiment requires a classical neural network baseline that does not rely on ODE-based frameworks to compare the performance of the proposed framework. The Transformer, a widely adopted neural network framework, is effective in handling sequential data and has been applied across various domains. Therefore, it serves as an appropriate baseline framework for comparison. It can be expressed as follows:
    \begin{equation}
        \hat{y} = \text{TransformerEncoder}(y + \text{PE})
    \end{equation}
    where $\text{PE}$ stands for positional encoding, the TransformerEncoder is a framework consisting of multiple layers of self-attention mechanisms and feed-forward networks. By adding positional encodings to the input data, the Transformer framework can incorporate sequential position information, enabling it to perform time series prediction.
    \item \textbf{CSODE}: The CSODE is used as an unaugmented baseline model to validate the correctness of Lemma \ref{lemma_training_nozero}. Its structure is similar to \eqref{LA-node:1}, but without the augmented state. Its form can be expressed as
    \begin{equation}
        \dot{y}(t) = a_0y(t) 
+ \sum_{i=1}^{H}(a_i)^\top 
f_i\left(s_iy(t)\right).
    \end{equation}
    \item \textbf{ANODE}: The ANODE is a general-purpose augmentation framework chosen as a baseline model due to its status as the most representative augmentation framework currently available. It incorporates an augmented dimension but does not adopt the CSODE structure, instead relying on the classic NODE architecture. The expression is as follows:
    \begin{equation}
        \frac{d}{dt} \left[ \begin{array}{c} y(t) \\ a(t) \end{array} \right] = f \left( \left[ \begin{array}{c} y(t) \\ a(t) \end{array} \right], t \right)
    \end{equation}

\end{itemize}

 Comparative experiments with these frameworks are conducted to demonstrate the performance advantages of the proposed framework.

\subsection{Industrial Systems for Experiments}
The frameworks were validated using physical experimental data from an IPMSM model (belonging to a subclass of IPMSMs), and a simulated DES model.

\subsubsection{IPMSM Drive Model}\label{sec:PMSM}

The IPMSM is a typical strongly nonlinear and highly coupled electromechanical system. Owing to its high power density, high efficiency, and superior speed regulation performance, it has been widely employed in high-performance drive systems. In the $d-q$ reference frame, its dynamic behavior can be formulated as:
\begin{equation}\label{eq.12}
\left\{\begin{aligned}
   &\dot{i}_d=-\frac{R_s}{L_d}{i}_{d}+\frac{\omega_eL_q}{L_d}i_q+\frac{1}{L_d}u_d+ \Delta \theta_d(i_d,i_q), \\
   &\dot{i}_q=-\frac{R_s}{L_q}{i}_{q}-\frac{\omega_eL_d}{L_q}i_d+\frac{1}{L_q}u_q-\frac{\omega_e\psi_f}{L_q}+\Delta \theta_q(i_d,i_q), 
   \end{aligned} \right.
\end{equation}
where $R_s$ is the stator resistance; $\omega_e$ is the electrical angular velocity; $\psi_f$ is the permanent-magnet (PM) flux linkage; $i_d, i_q$ are the stator phase currents; $u_d, u_q$ are the terminal voltages; and $L_d, L_q$ represent the $dq$-axis inductances. The terms $\Delta \theta_d(\cdot)$ and $\Delta \theta_q(\cdot)$ represent cumulative model uncertainties and lumped nonlinearities. The mechanical dynamics of the motor, specifically the evolution of the electrical angular velocity $\omega_e$, are governed by
\begin{equation}\label{eq.omega}
\dot{\omega}_e = \frac{n_p}{J}(T_e - T_L) - \frac{B}{J} \omega_e,
\end{equation}
where the electromagnetic torque $T_e$ satisfies
\begin{equation}
T_e = \frac{3n_p}{2} \left[ \psi_f i_q + (L_d - L_q)i_d i_q \right].
\end{equation}
In these expressions, $n_p$ denotes the number of pole pairs, $J$ is the moment of inertia, $B$ is the viscous friction coefficient, and $T_L$ is the load torque. 

In this experiment, the IPMSM system is reduced in dimension via the mapping 
$h_1 \colon \mathbb{R}^3 \to \mathbb{R}^2$, shown as follows: 
\begin{equation}\label{pmsmd}
\begin{aligned}
y_a &= h_1(x_a), \\
x_a &=
\begin{bmatrix}
i_d \\ i_q \\ \omega_e
\end{bmatrix}
\in \mathbb{R}^3, \quad
y_a =
\begin{bmatrix}
i_d \\ i_q
\end{bmatrix}
\in \mathbb{R}^2.
\end{aligned}
\end{equation}

The Jacobian matrix of the reduction mapping is given by
\begin{equation}
J_{h_1}(x_a)
=
\frac{\partial h_1}{\partial x_a}
=
\begin{bmatrix}
1&0&0\\
0&1&0
\end{bmatrix}.
\end{equation}
The reduced-order velocity is determined by
\begin{equation}
\dot{y}_a=J_{h_1}(x_a)F(x_a).
\end{equation}

Since $\omega_e$ is eliminated during the dimension reduction, different full-order states with different $\omega_e$ values may correspond to the same reduced observation $y_a^\star$, such as
\begin{equation}
h_1(x_a^{(1)})=h_1(x_a^{(2)})=y_a^\star,\quad J_{h_1}(x_a^{(1)})F(x_a^{(1)})\neq J_{h_1}(x_a^{(2)})F(x_a^{(2)}).
\end{equation}

Therefore, the IPMSM system exhibits a two-trajectory singularity. The corresponding velocity ambiguity set is defined as
\begin{equation}
\mathcal{V}(y_a^\star)
=
\{J_{h_1}(x_a)F(x_a)\mid h_1(x_a)=y_a^\star\}.
\end{equation}

As the two velocity trajectories exist, their velocity vectors span a one-dimensional subspace:
\begin{equation}
r(y_a^\star)
=
\dim(\operatorname{span}(\mathcal{V}(y_a^\star)))
=1.
\end{equation}

By Theorem~\ref{co}, the minimum augmentation dimension satisfying the representational requirement is therefore selected as $k=1$.
The values of key physical parameters of the IPMSM utilized in this study are summarized in Table~\ref{tab:PMSM}.

\begin{table}[!htb]
\centering
\caption{Key parameters of the IPMSM system\label{tab:PMSM}}
\renewcommand{\arraystretch}{1.2} 
\begin{tabular}{lcc}
\toprule
Parameter & Symbol & Value [Unit] \\
\midrule
Rated power         & $P$            & 150 [W] \\
Rated speed         & $N_r$          & 1000 [rpm] \\
Rated load          & $T_L$          & 1.2 [N$\cdot$m] \\
Rated voltage       & $U_r$          & 24 [V] \\
Rated current       & $I_r$          & 7.5 [A] \\
Number of pole pairs & $n_p$         & 4 [--] \\
Rotational inertia  & $J$            & $7.06 \times 10^{-5}$ [kg$\cdot$m$^2$] \\
PM flux linkage     & $\psi_f$       & 0.0256 [Wb] \\
Stator resistance   & $R_s$          & 0.72 [$\Omega$] \\
$d$-axis inductance & $L_d$          & 0.2 [mH] \\
$q$-axis inductance & $L_q$          & 0.4 [mH] \\
\bottomrule
\end{tabular}
\end{table}

\subsubsection{Distributed Energy System}

 DES is a dynamic system composed of photovoltaic generation units, energy storage units, and dynamic loads, capturing the essential characteristics of multi-energy coupling and energy balance in practical microgrid or islanded operation scenarios. In this study, the DES is used as another benchmark system to evaluate further the modeling and generalization performance of the proposed LA-NODE framework in scenarios involving trajectory singularity. Its continuous-time dynamics can be expressed by the following equations:

\begin{equation}
\left\{
\begin{aligned}
\dot{x}_1 &= -\frac{\eta_{\textrm{bat}} \cdot P_{\textrm{bat}}}{3600 \cdot E_{\textrm{bat,max}}}, \\
\dot{x}_2 &= \frac{x_3 + P_{\textrm{bat}} - x_4 - D(x_2 - \omega_0)}{M}, \\
\dot{x}_3 &= \frac{P_{\textrm{PV,ref}} - x_3}{T_{\textrm{PV}}}, \\  
\dot{x}_4 &= \frac{P_{\textrm{load,ref}} - x_4}{T_{\textrm{load}}}. 
\end{aligned}
\right.
\end{equation}
Here, $x_1$ denotes the state of charge of the energy storage unit, $x_2$ represents the system angular frequency ($\omega$), and $x_3$ and $x_4$ correspond to the actual output power of the photovoltaic unit and the actual load power, respectively.
The control input $P_{\text{bat}}$, regulated by a proportional-integral controller to maintain frequency stability, is designed as
\begin{equation}P_{\text{bat}}(t) = -K_p \Delta\omega(t) - K_i \int_{0}^{t} \Delta\omega(\tau) d\tau,\end{equation}
where $\Delta\omega = x_2 - \omega_0$ denotes the frequency deviation from the nominal angular frequency $\omega_0$. To ensure the physical plausibility of system operation, constraints are imposed on $x_1$ and $x_2$ during the simulation: 
\begin{equation}
 0 \le x_1(t) \le 1, \quad
 0.9  \times \omega_0 \le x_2(t) \le 1.1 \times \omega_0.
\label{eq:DES_constraints}
\end{equation}
In the simulation, the values of system parameters and operational constraints are selected as shown in Table~\ref{tab:DES_params}.

In this experiment, the internal state of the DES system is observed via the mapping 
$h_2 \colon \mathbb{R}^4 \to \mathbb{R}^2$, shown as follows: 

\begin{equation}\label{desd}
\begin{aligned}
y_b &= h_2(x_b), \\
x_b &=
\begin{bmatrix}
x_1 \\ x_2 \\ x_3 \\x_4
\end{bmatrix}
\in \mathbb{R}^4, \quad
y_b =
\begin{bmatrix}
x_3 \\ x_4
\end{bmatrix}
\in \mathbb{R}^2.
\end{aligned}
\end{equation}

The Jacobian matrix of the reduction mapping is given by
\begin{equation}
J_{h_2}(x_b)
=
\frac{\partial h_2}{\partial x_b}
=
\begin{bmatrix}
0&0&1&0\\
0&0&0&1
\end{bmatrix}.
\end{equation}
The reduced-order velocity is determined by
\begin{equation}
\dot{y}_b
=
J_{h_2}(x_b)F(x_b).
\end{equation}

Similar to the IPMSM case, multiple internal states  may correspond to the same reduced observation, 
\begin{equation}
h_2(x_b^{(1)})
=
h_2(x_b^{(2)})
=
h_2(x_b^{(3)})
=
y_b^\star.
\end{equation}

Therefore, for the DES system, we have
\begin{equation}
\mathcal{V}(y_b^\star)
=
\{J_{h_2}(x_b)F(x_b)\mid h_2(x_b)=y_b^\star\},\quad r(y_b^\star)
=
\dim(\operatorname{span}(\mathcal{V}(y_b^\star)))
=2.
\end{equation}

Under Theorem~\ref{co}, the minimum augmentation dimension satisfying the representational requirement is therefore selected as $k=2$.

\begin{table}[!htb]
\centering
\caption{Key parameters of the DES \label{tab:DES_params}}
\renewcommand{\arraystretch}{1.2} 
\begin{tabular}{lcc}
\toprule
Parameter & Symbol & Value [Unit] \\
\midrule
Maximum battery capacity    & $E_{\text{bat,max}}$  & 5000 [kWh] \\
Charge/discharge efficiency & $\eta_{\text{bat}}$   & 0.95 \\
Inertia constant            & $M$                   & 10 [s] \\
Damping coefficient         & $D$                   & 1 [p.u.] \\
Rated angular frequency     & $\omega_0$            & $100\pi$ [rad/s] \\
PV dynamic time constant    & $T_{\text{PV}}$       & 5 [s] \\
Load dynamic time constant  & $T_{\text{load}}$     & 3 [s] \\
Proportional / Integral gain & $K_P / K_I$          & 20 / 50 \\
Max. battery power          & $P_{\text{bat}}^{\text{max}}$ & 1500 [kW] \\
Reference PV power range    & $P_{\text{PV,ref}}$   & [0, 1800] [kW] \\
Reference load power range  & $P_{\text{load,ref}}$ & [1500, 2500] [kW] \\
\bottomrule
\end{tabular} 
\end{table}

\subsection{Experimental Results and Analysis}
\subsubsection{Experiment on IPMSM Drive}
A dataset comprising $300$ pairs of $\{i_d, i_q\}$ was obtained by subjecting the system to a reverse step signal. 
The IPMSM experimental data used in this work were collected from the same experimental platform as that in \cite{mei2025learning}. 
The raw current trajectories contain noise-like disturbances caused by uncertain terms in the IPMSM model. 
Following the preprocessing procedure in \cite{mei2025learning}, Gaussian regression was employed to smooth the measured signals while preserving their dynamic characteristics. 
The processed data were subsequently normalized and used for model training and evaluation.
Subsequently, the four experimental frameworks were evaluated using the processed data. In the specific IPMSM system, the trajectory singularities exhibit only two distinct velocity directions. By Theorem \ref{co}, a one-dimensional augmentation is sufficient for trajectory separation in this experiment. 

\begin{figure}
    \centering
    \includegraphics[width=0.7\textwidth]{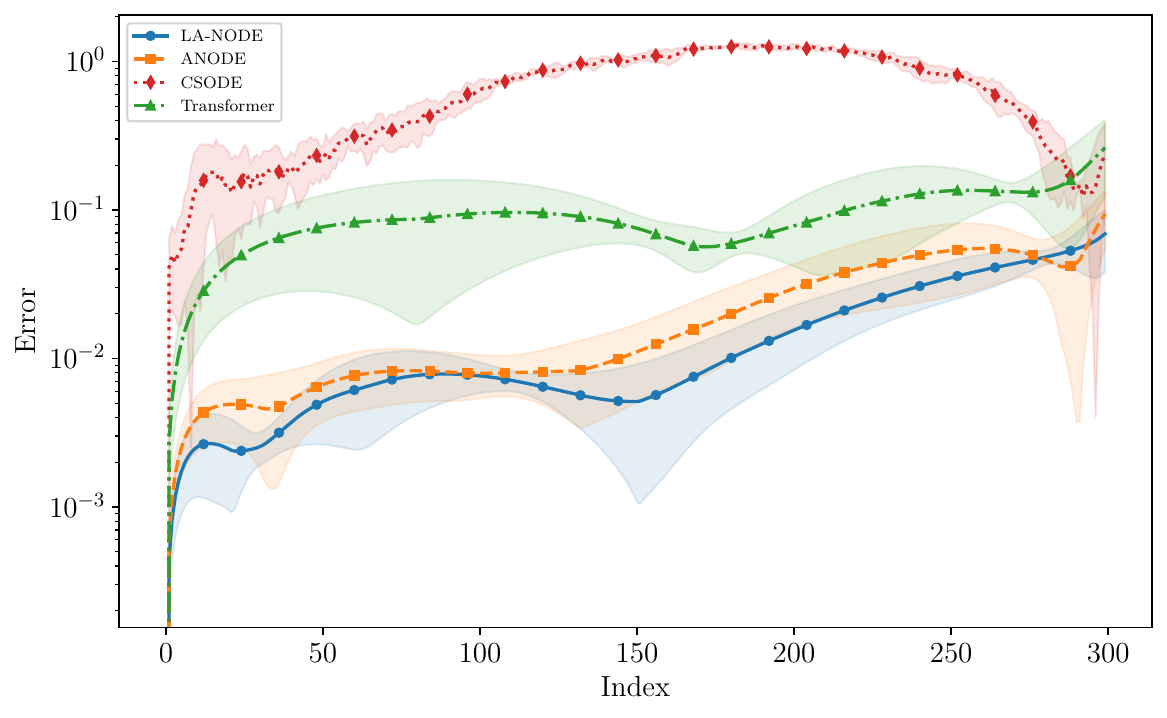}
    \caption{Comparison of trajectory-wise errors ($Error=\Big\|\sqrt{(i_{d,true}(i) - i_{d,pred}(i))^2 + (i_{q,true}(i) - i_{q,pred}(i))^2}\Big\|$) across four frameworks over the full dataset for IPMSM. Index represents the data point position in the IPMSM dataset.}
    \label{fig.PMSM1}
\end{figure}

As shown in Fig.~\ref{fig.PMSM1}, the four frameworks exhibit distinct trajectory-wise error levels. In terms of overall error level, the proposed LA-NODE achieves the lowest error, remaining within [0.001,0.01]. In contrast, ANODE yields a slightly higher overall error, mostly on the order of 0.01, while the traditional Transformer remains at around 0.1. CSODE achieves the highest error, remaining at the level of 1 for most indices with fluctuations. These results suggest that ODE-based approaches, such as LA-NODE and ANODE, are more suitable for modeling continuous dynamical systems in this setting. In this experiment, CSODE fails to capture dynamics with singularities.

Meanwhile, the variance distribution from the shaded regions shows that CSODE has the smallest fluctuation; however, due to its large error, this low variability is not indicative of stable modeling performance. LA-NODE and ANODE exhibit narrow fluctuations, demonstrating stronger robustness and consistency. By contrast, Transformer yields the largest variance, implying it is less stable.

\begin{figure}
    \centering
    \subfigure[]{
        \includegraphics[width=0.48\columnwidth]{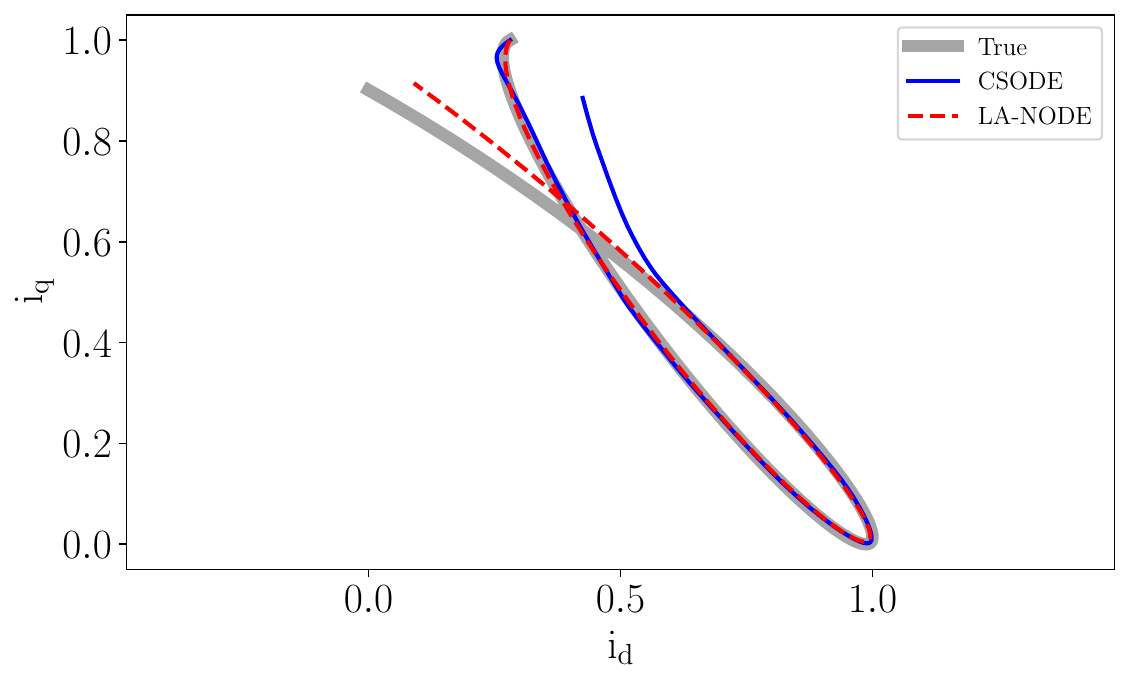}
        \label{fig.PMSM21}
    }
    \hfill
    \subfigure[]{
        \includegraphics[width=0.48\columnwidth]{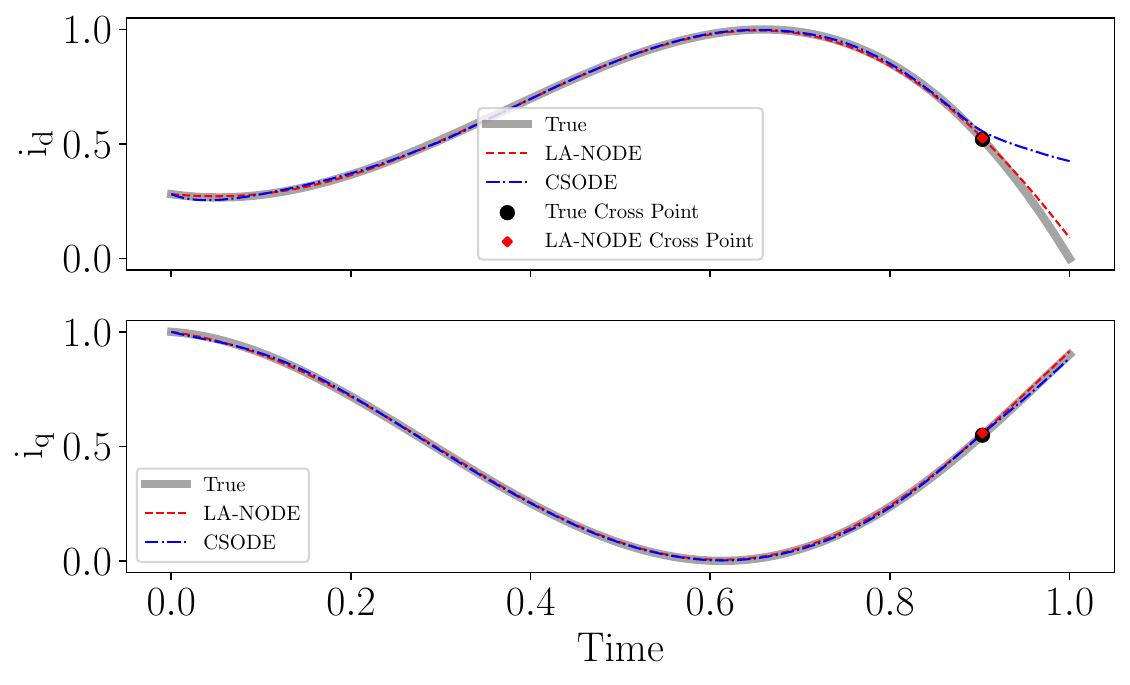}
        \label{fig.PMSM22}
    }
    \caption{Performance comparison between CSODE and LA-NODE based on IPMSM trajectories: (a) Phase portraits, (b) Time-domain responses. }
    \label{fig.PMSM2}
\end{figure}

To further validate Lemma \ref{lemma_training_nozero}, the learning performances of LA-NODE and CSODE are compared from the perspectives of phase-space trajectories and time responses in Fig.~\ref{fig.PMSM2}, respectively. Overall, both methods approximate system behavior well before singularities. However, as the system approaches the singular region, the predictions of CSODE gradually deviate from the true trajectories, exhibiting significant errors in the direction of its vector field.

In contrast, LA-NODE maintains consistent performance across the entire admissible state space. Fig.~\ref{fig.PMSM22} illustrates the presence of singularities along three trajectories. It can be observed that, although there exists a slight discrepancy between the singular points identified by LA-NODE and those of the true trajectories, the learned vector field remains close to the ground truth. This indicates that LA-NODE preserves the correct dynamic evolution and captures singular structures. These results also provide empirical support for Lemma \ref{Augmented}.

\begin{figure}
    \centering
    \subfigure[]{
        \includegraphics[width=0.48\columnwidth]{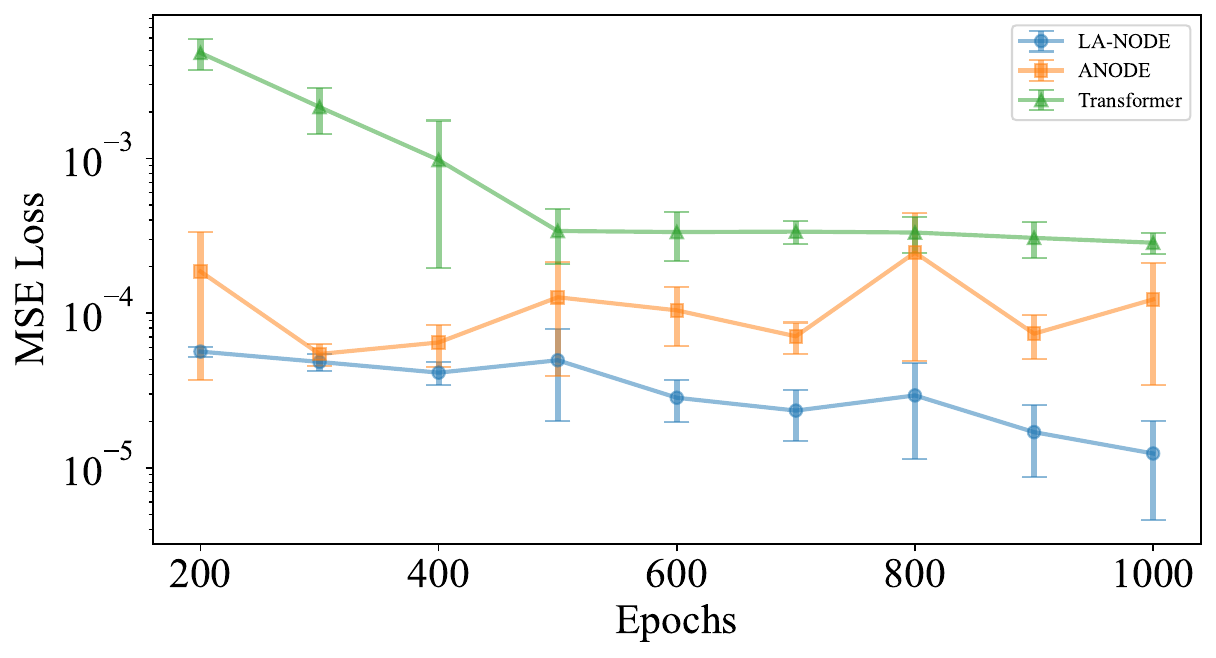}
        \label{fig.PMSM31}
    }
    \hfill
    \subfigure[]{
        \includegraphics[width=0.48\columnwidth]{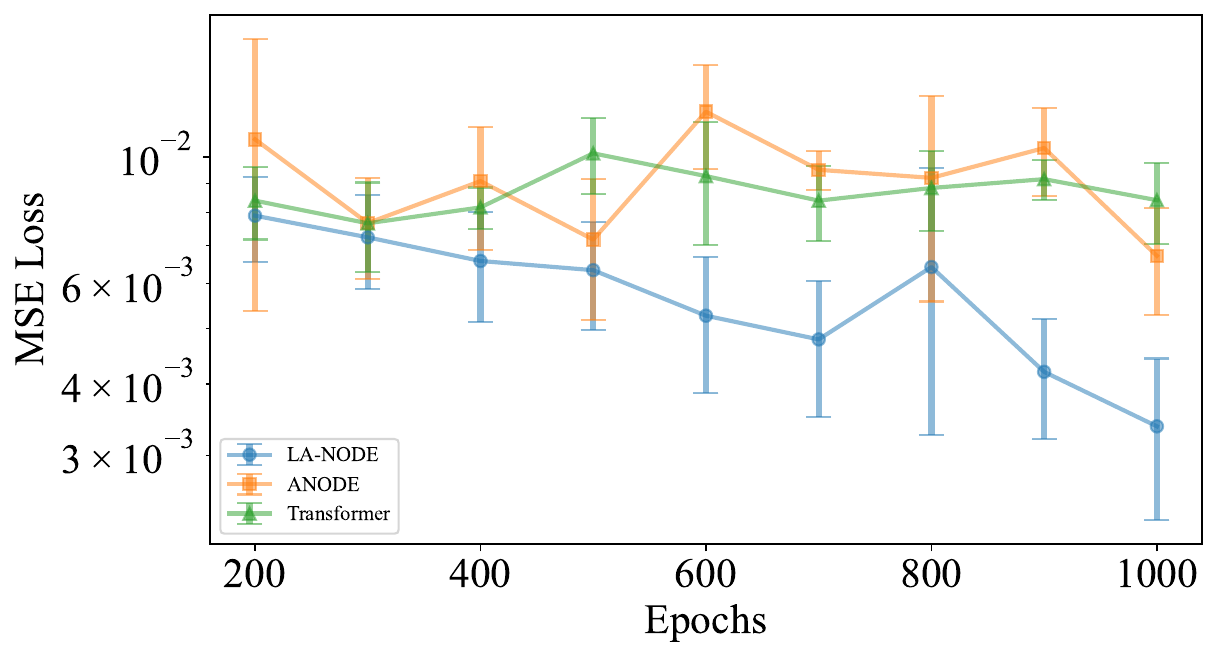}
        \label{fig.PMSM32}
    }
    \caption{Comparison of MSE Loss across different frameworks with epochs for IPMSM: (a) Training set, (b) Testing set.}
    \label{fig.PMSM3}
\end{figure}


Fig. \ref{fig.PMSM3} illustrates the evolution of the MSE Loss for different frameworks across both the training and testing sets.  LA-NODE consistently outperforms both ANODE and Transformer in terms of MSE Loss. As shown in Fig. \ref{fig.PMSM31}, the MSE Loss of LA-NODE on the training set remains lower than that of the other two frameworks, ranging from $[0.0001,0.00001]$, while ANODE and Transformer exhibit MSE Loss values in the range of $[0.001,0.0001]$. As shown in Fig. \ref{fig.PMSM32}, on the testing set, the MSE Loss of LA-NODE is initially comparable to that of ANODE and Transformer, but as the number of epochs increases, its MSE Loss significantly decreases and remains lower than that of the other two frameworks.

In summary, these results confirm that LA-NODE is more effective in modeling the reduced-order system in IPMSM with singularity in the settings, achieving superior modeling performance compared with representative baseline frameworks.

\subsubsection{Experiment on DES}
\begin{figure*}
\centering

\subfigure[]{
    \includegraphics[width=0.22\textwidth]{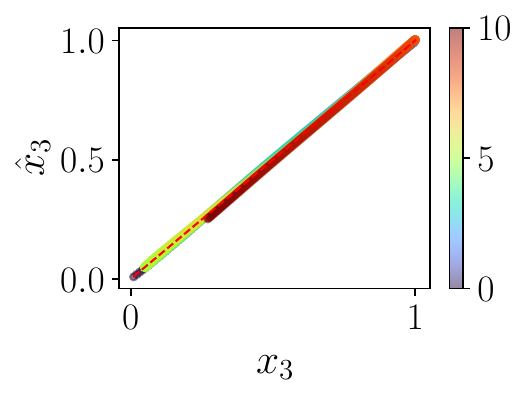}
}
\hfill
\subfigure[]{
    \includegraphics[width=0.22\textwidth]{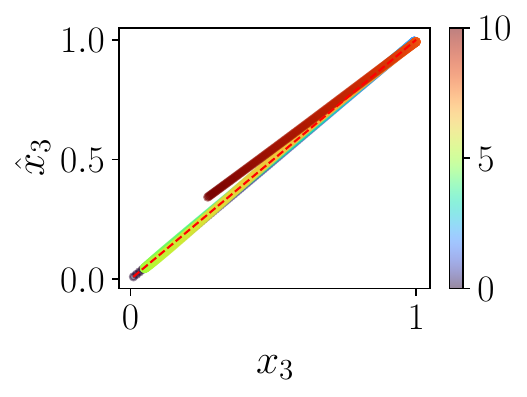}
}
\hfill
\subfigure[]{
    \includegraphics[width=0.22\textwidth]{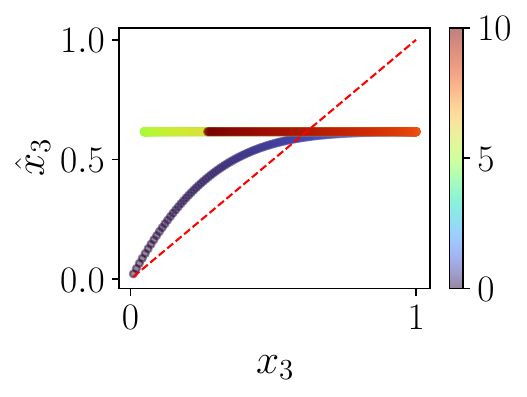}
}
\hfill
\subfigure[]{
    \includegraphics[width=0.22\textwidth]{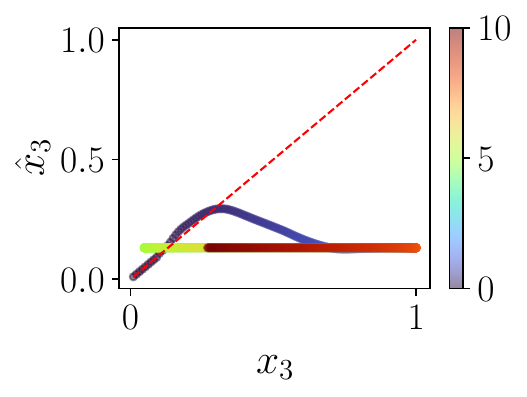}
}

\vspace{0.5em}

\subfigure[]{
    \includegraphics[width=0.22\textwidth]{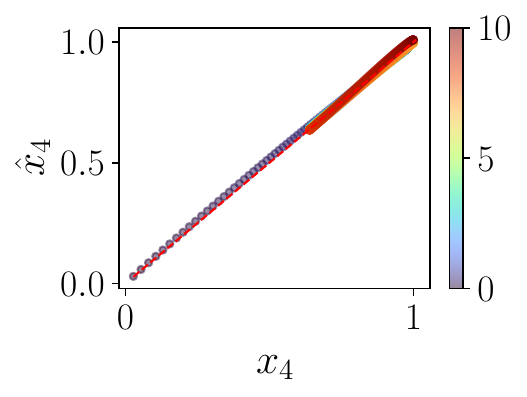}
}
\hfill
\subfigure[]{
    \includegraphics[width=0.22\textwidth]{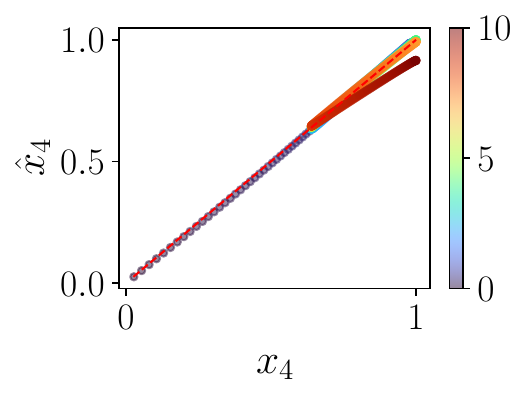}
}
\hfill
\subfigure[]{
    \includegraphics[width=0.22\textwidth]{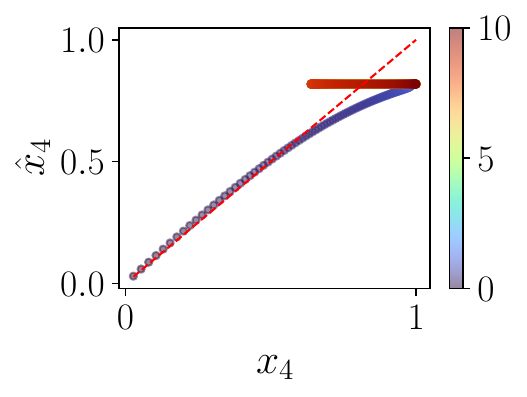}
}
\hfill
\subfigure[]{
    \includegraphics[width=0.22\textwidth]{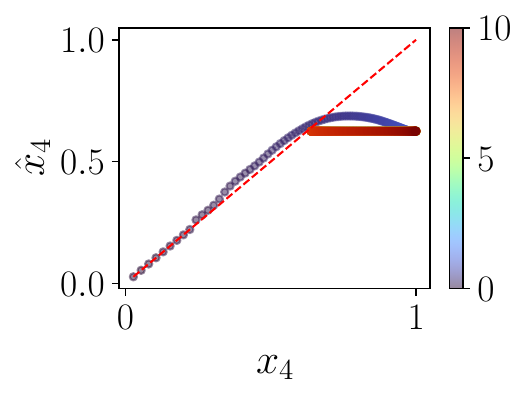}
}

\caption{Comparison of state variables for DES (the color variation indicates the progression of time): (a) LA-NODE, (b) ANODE, (c) CSODE, (d) Transformer, for $x_3$; (e) LA-NODE, (f) ANODE, (g) CSODE, (h) Transformer, for $x_4$.}
\label{fig.des1}
\end{figure*}

In this experiment, time-series data of $\{x_3,x_4\}$ were generated based on the aforementioned simulation model to construct the dataset for model training and testing. The simulation lasts 10s with a sampling interval of 0.01s, yielding 1000 samples. Despite the foundational nature of these variables, their trajectories exhibit singularities, posing a challenge for conventional frameworks. In the proposed architecture, the predicted $x_3$ and $x_4$ are incorporated into the dynamics of $x_2$ to support the PI controller’s frequency regulation.

\begin{table*}[htbp]
\centering
\caption{Comparison of four dynamic modeling frameworks for DES}
\label{tab.desex}

\setlength{\tabcolsep}{4pt}
\renewcommand{\arraystretch}{1.25} 
\begin{tabular}{ccccc}

\toprule
\textbf{Model} & \textbf{MAE} & \textbf{RMSE} & \textbf{Max Error} & \textbf{Std Error} \\
\hline

LA-NODE (ours) 
& \textbf{0.0075 $\pm$ 0.0037}
& \textbf{0.0097 $\pm$ 0.0049}
& \textbf{0.0270 $\pm$ 0.0138}
& \textbf{0.0061 $\pm$ 0.0033} \\

CSODE 
& 0.1771 $\pm$ 0.0053
& 0.2266 $\pm$ 0.0046
& 0.5819 $\pm$ 0.0095
& 0.1404 $\pm$ 0.0008 \\

Transformer 
& 0.2625 $\pm$ 0.0404
& 0.3517 $\pm$ 0.0422
& 0.8586 $\pm$ 0.0498
& 0.2333 $\pm$ 0.0240 \\

ANODE 
& 0.0171 $\pm$ 0.0085
& 0.0236 $\pm$ 0.0114
& 0.0858 $\pm$ 0.0484
& 0.0161 $\pm$ 0.0083 \\
\toprule

\end{tabular}
\end{table*}

Specifically, for the DES under consideration, three distinct velocity directions are observed at a singularity. By Theorem \ref{co}, an augmented dimension of $k \ge 2$ is required for trajectory separation. As a result, a two-dimensional augmented space was employed for this experiment. 

Fig.~\ref{fig.des1} illustrates significant differences in the predictive performance of various models on the two state variables ($x_3$ and $x_4$). Ideally, scatter points along with $y=x$, indicating high agreement between predictions and ground truth. Among the comparative frameworks, LA-NODE exhibits the most compact distribution, with points tightly aligned along the diagonal, indicating high prediction accuracy. In contrast, although ANODE shows a more dispersed scatter distribution, its points remain close to the reference diagonal, indicating relatively high modeling capability. However, the Transformer and CSODE frameworks exhibit more pronounced deviations from the diagonal, suggesting larger prediction errors. In particular, during the later stages of the time series, both models exhibit noticeable performance degradation and are unable to continuously track the system's evolution. 

\begin{figure}
    \centering
    \includegraphics[width=0.8\textwidth]{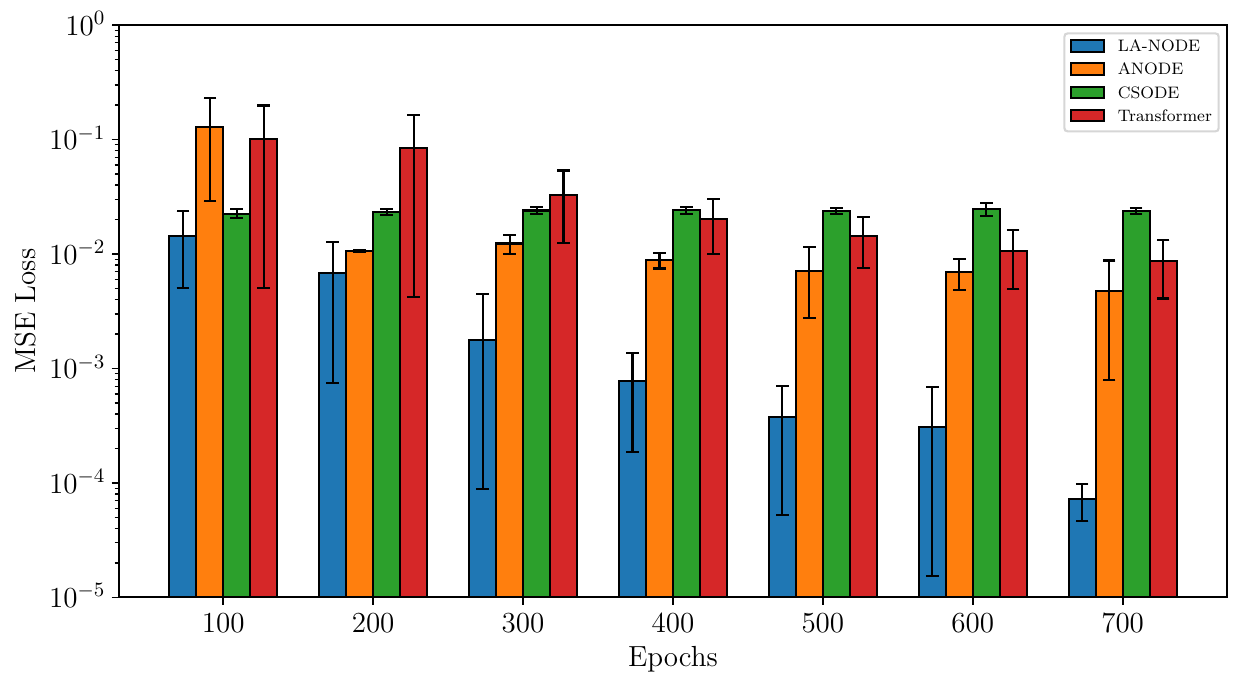}
    \caption{Comparison of MSE Loss for the testing set across four frameworks for DES.}
    \label{fig.des2}
\end{figure}

Fig.~\ref{fig.des2} illustrates the evolution of the test MSE for different frameworks. Except for CSODE, all frameworks show a decreasing trend. Among them, LA-NODE consistently achieves the lowest error, rapidly decreasing from 0.01 to below 0.0001. In contrast, ANODE starts with a relatively large error, which gradually decreases and eventually stabilizes around 0.0003. The Test MSE of CSODE remains around 0.01 with slight fluctuations, suggesting a clear performance bottleneck. The Transformer reduces error gradually but remains within [0.01,0.001], indicating limited precision.

As shown in Table~\ref{tab.desex}, four metrics are adopted, including the MAE, RMSE, Maximum Error (Max Error), and the Standard Deviation of Error (Std Error).  Among them, Max Error is used to measure the worst-case prediction deviation, while Std Error reflects the dispersion of errors and is employed to assess the framework’s stability. The experimental results indicate that LA-NODE achieves the best performance across all evaluation metrics. Compared with the second-best framework ANODE, LA-NODE reduces MAE and RMSE by approximately 56.1\% and 58.9\%, respectively. In addition, the reductions in Max Error and Std Error reach about 68.5\% and 62.1\%, respectively. Meanwhile, the standard deviations of all metrics for LA-NODE are consistently smaller than those of the competing methods except for CSODE, further indicating its superior stability and consistency. 

In summary, the results consistently demonstrate that LA-NODE achieves superior performance in terms of accuracy, stability, and generalization, making it more effective for learning DES dynamics compared with existing approaches.

\section{Sensitivity Analysis of the Augmented Dimension}\label{sec6}
As shown above, learning the dynamics of the DES requires an augmented dimension of LA-NODEs satisfying $k \ge 2$. To evaluate the significance of this theoretical constraint, this section presents a sensitivity analysis on the model performance, examining the value of $k$ from 1 to 5.

\begin{figure}
    \centering
    \includegraphics[width=0.8\textwidth]{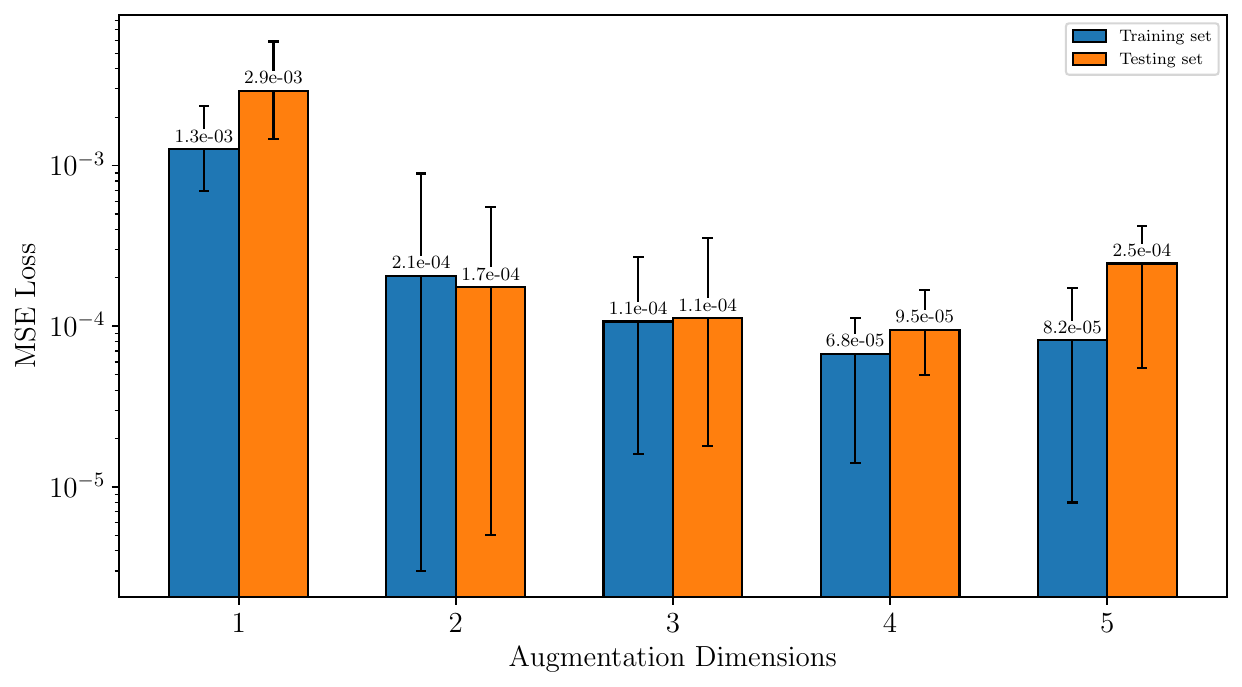}
    \caption{Comparison of MSE Loss (on logarithmic scale) for training set and testing set across different augmentation dimensions.}
    \label{fig.db}
\end{figure}

Fig.~\ref{fig.db} illustrates the learning performance of the framework under different augmentation dimensions. The error generally decreases as $k$ increases, except for $k=4 \to k=5$. However, this result only reflects the absolute error level and does not capture the actual contribution of each augmentation stage to performance improvement. To further quantitatively analyze the impact of augmentation dimensions on model performance, a contribution rate metric based on global error reduction is introduced. 
Specifically, let $E_k$ denote the error under augmentation dimension $k$. The total error reduction is defined as
\begin{equation}
\Delta E_{\mathrm{total}}=E_{\max}-E_{\min},
\end{equation}
where $E_{\max}$ and $E_{\min}$ are the maximum and minimum errors among all augmentation dimensions, respectively. The contribution ratio of each augmentation stage is defined as
\begin{equation}
C_k=
\frac{E_k-E_{k+1}}
{\Delta E_{\mathrm{total}}}
\times100\%.
\end{equation}

For training error, the transition from $k=1 \to k =2$ contributes 88.5\% of the total reduction, while $k=2 \to k =3$ and $k=3 \to k =4$ contribute only 8.1\% and 3.4\%, respectively. For test error, this effect is more pronounced: the $k=1 \to k =2$ stage contributes 97.4\%, while subsequent stages contribute only 2.1\% and 0.5\%. In addition, when the augmentation dimension increases to 5, the test error increases, showing a negative contribution.


To further investigate the influence of the augmentation dimension, the training time is also evaluated under different augmentation dimensions. Table~\ref{tab:aug_time} shows the average total training time over five independent runs for 700 epochs. It can be observed that the training time increases with the augmentation dimension, since a larger augmented space introduces additional model parameters and optimization complexity.

\begin{table}[b]
\centering
\caption{Average total training time (sec) over 5 runs for 700 epochs, varying $k$.}
\label{tab:aug_time}
\begin{tabular}{c c c c c c}
\toprule
$k$ & 1&2 & 3 & 4 & 5 \\
\midrule
Time (sec) &637.2& 638.2 & 648.9 & 684.0 & 740.2 \\
\bottomrule
\end{tabular}
\end{table}

\begin{figure}
    \centering
    \includegraphics[width=0.8\textwidth]{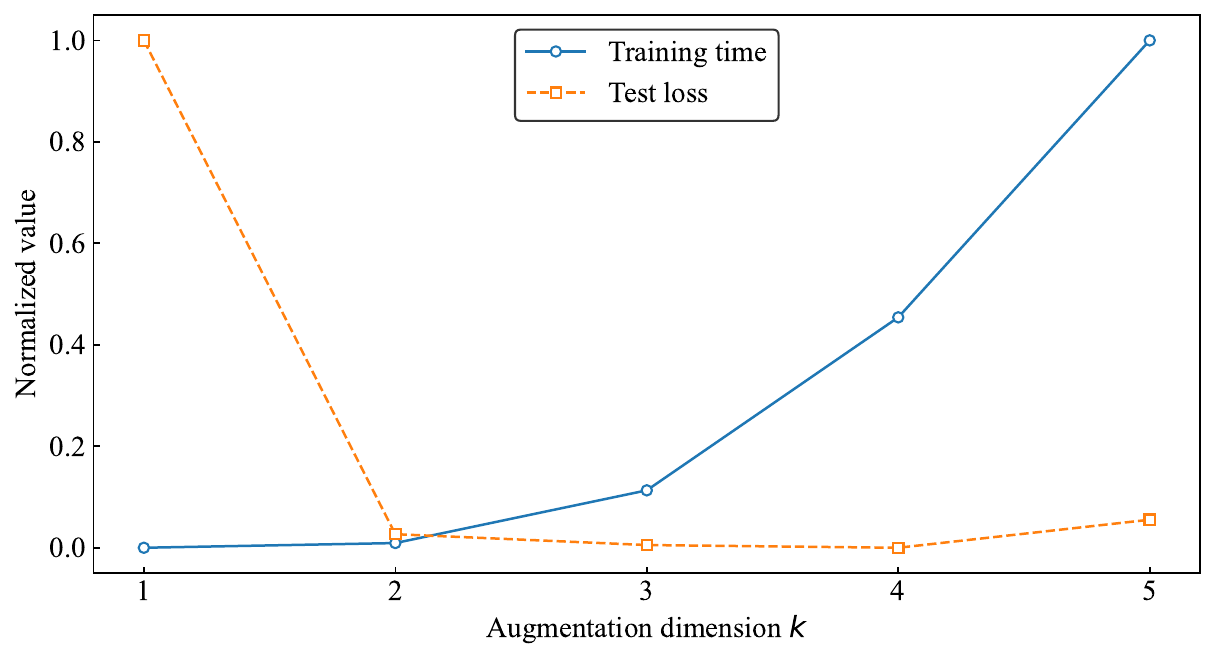}
    \caption{Trade-off analysis between computational cost and test loss under different augmentation dimensions.}
    \label{fig.tradeoff}
\end{figure}

Since training time and test loss have different scales, min-max normalization is applied to both quantities:
\begin{equation}
\hat{T}_k=\frac{T_k-T_{\min}}{T_{\max}-T_{\min}},\qquad
\hat{L}_k=\frac{L_k-L_{\min}}{L_{\max}-L_{\min}}, 
\end{equation}
where $\hat{T}_k$ and $\hat{L}_k$ denote the normalized computational cost and test loss, respectively. The corresponding trade-off result is presented in Fig.~\ref{fig.tradeoff}. It shows that the major performance improvement is achieved when increasing the augmentation dimension from $k=1$ to $k=2$, whereas further increasing $k$ only provides a marginal reduction in test loss while introducing additional computational cost. Therefore, $k=2$ achieves a favorable balance between representation capability and computational efficiency and is adopted in this setting.

It is also noteworthy that increasing the augmentation dimension does not always lead to a monotonic improvement in test loss. In particular, the test loss slightly increases when $k$ increases from $4$ to $5$. This phenomenon is mainly attributed to the fact that once the augmentation dimension satisfies the required velocity ambiguity rank $k\geq r(y^\star)$, additional latent dimensions may not provide further useful information for representing the reduced-order dynamics. For the DES, the maximum velocity ambiguity rank is
$\max_{y^\star}r(y^\star)=2$.

Therefore, augmentation dimensions larger than this requirement may introduce redundant latent variables and increase the optimization complexity, resulting in minor performance fluctuations. This also indicates that selecting the minimum feasible augmentation dimension is expected to provide a better balance between representation capability and computational efficiency.


\section{Conclusion}\label{sec7}
This paper proposes the Latent-Augmented Neural Ordinary Differential Equations (LA-NODEs) framework, which addresses an important limitation of standard ControlSynth NODEs (CSODEs) by augmenting the latent state dimension to enable the learning of reduced-order dynamical systems with trajectory singularities. A key theoretical contribution was the establishment of an explicit criterion for determining the minimum required augmentation dimension, which ensured a balance between model expressivity and computational tractability. Experiments on multiple industrial systems further demonstrate that LA-NODEs consistently outperforms state-of-the-art baselines in accurately capturing complex reduced-order dynamics.

When applying LA-NODE to a new reduced-order system, the augmentation dimension should be selected per the maximum velocity ambiguity rank among singularities, following the criterion in Theorem~\ref{co}. In practice, the minimum feasible augmentation dimension satisfying Theorem~\ref{co} is required, as an excessively large augmentation dimension may introduce unnecessary computational complexity and does not necessarily lead to further improvement in modeling performance.

Despite these advancements, several challenges remain and require further investigation. The effectiveness of LA-NODEs relies on the existence of a finite-dimensional latent space that can resolve the velocity ambiguity caused by state projection. For systems with severe information loss or highly complex hidden dynamics, state augmentation alone may be insufficient to construct an accurate representation of the underlying dynamics. In this study, the effectiveness of the LA-NODEs framework was primarily validated on low-dimensional manifolds; its scalability and generalization to high-dimensional state spaces remained to be fully explored. Furthermore, while the model exhibited robustness in the tested scenarios, its stability under stochastic perturbations or out-of-distribution disturbances required more rigorous quantification. 

Future research will focus on extending the framework to high-dimensional complex systems through structural optimization and synergistic integration with advanced control-theoretic methods. Efforts can be directed toward enhancing the model’s resilience under extreme disturbances and optimizing its computational efficiency for real-time control.



\bibliographystyle{elsarticle-harv}
\bibliography{cas-refs}

\end{document}